\newif\ifdraft
\draftfalse
\documentclass[a4paper]{article}
\usepackage[margin=3.5cm]{geometry}
\usepackage{amsmath, amssymb, amsthm}
\usepackage{natbib}
\usepackage[colorlinks,
            linkcolor=blue,
            citecolor=blue,
            urlcolor=magenta,
            linktocpage,
            plainpages=false]{hyperref}
\usepackage{cleveref}
\usepackage{authblk}
\usepackage{mathtools}
\usepackage{thmtools}
\usepackage{nicefrac}
\usepackage{bm}
\usepackage{xcolor}
\usepackage{soul}
\usepackage{enumitem}

\usepackage{tikz}
\usetikzlibrary{arrows.meta,calc,decorations.pathreplacing}

\newcommand{\eps}{\varepsilon}

\newenvironment{proofsketch}
{\noindent\textit{Proof Idea.}\ }
{\hfill$\square$\par\bigskip}

\newtheorem{theorem}{Theorem}
\newtheorem{definition}{Definition}

\newtheorem{lemma}{Lemma}

\title{Structured vs. Unstructured Pruning: \\An Exponential Gap}
\author[1]{Davide Ferré}
\author[1]{Frédéric Giroire}
\author[2]{Frederik Mallmann-Trenn}
\author[1]{Emanuele Natale}
\affil[1]{Université Côte d'Azur, CNRS, Inria, I3S, France}
\affil[2]{Department of Informatics, King’s College London}

\begin{document}

\maketitle

\begin{abstract}
The Strong Lottery Ticket Hypothesis (SLTH) states that large, randomly initialized neural networks contain sparse subnetworks capable of approximating a target function at initialization without training, suggesting that pruning alone is sufficient. Pruning methods are typically classified as unstructured, where individual weights can be removed from the network, and structured, where parameters are removed according to specific patterns, as in neuron pruning. Existing theoretical results supporting the SLTH rely almost exclusively on unstructured pruning, showing that logarithmic overparameterization suffices to approximate simple target networks. In contrast, neuron pruning has received limited theoretical attention, despite its practical appeal for direct hardware speedups. 
In this work, we consider the problem of approximating a single bias-free ReLU neuron by pruning hidden units of a randomly initialized two-layer ReLU network, effectively isolating the intrinsic limitations of neuron pruning. 
We show that achieving an $\varepsilon$-approximation requires a starting network size of $\Omega(1/\varepsilon)$ for neuron pruning, whereas weight pruning succeeds with only $O(\log(1/\varepsilon))$ hidden units, revealing an exponential separation between the two approaches.
\end{abstract}

\section{Introduction}
Modern neural networks are typically trained in highly overparameterized regimes, often containing orders of magnitude more parameters than are seemingly necessary to represent the target function.
Despite this redundancy, empirical evidence suggests that such large models can be pruned aggressively after training, yielding sparse subnetworks that retain comparable performance.
Understanding the origin of this phenomenon has become a central topic in the theory of deep learning.

A prominent line of work addressing this question is the \emph{Lottery Ticket Hypothesis} (LTH), introduced in \cite{frankleLotteryTicketHypothesis2018}.
The LTH states that sufficiently large randomly initialized networks contain sparse subnetworks, called \emph{winning tickets}, which can achieve the performance of the original dense network, when trained in isolation.
Subsequent empirical works refined this perspective and proposed pruning strategies that identify such subnetworks efficiently \citep{zhouDeconstructingLotteryTickets2019,ramanujanWhatHiddenRandomly2020}.

These developments motivated an even stronger conjecture, known as the \emph{Strong Lottery Ticket Hypothesis} (SLTH), which asserts that winning tickets already exist at initialization, without requiring any training of the retained weights.
In other words, pruning alone suffices to obtain a performant subnetwork.
This formulation removes the need to analyze training dynamics and has enabled a series of rigorous existence proofs for fully connected and convolutional architectures~\citep{malachProvingLotteryTicket2020,Pensia,orseauLogarithmicPruningAll2020,ferbachGeneralFrameworkProving2022}.
Given a target network and $\varepsilon > 0$, a central question is how much overparameterization is needed for a randomly initialized network to contain, with high probability, a pruned subnetwork that $\varepsilon$-approximates the target network (see \Cref{def:eps-apx} for the precise metric). This dependence on $\varepsilon$ is critical in deep settings: since approximation errors compound across layers, the required accuracy must shrink with the target network’s width and depth. Consequently, a pruning scheme with linear dependence on $1/\varepsilon$ becomes prohibitively expensive, while the logarithmic dependence achieved by weight pruning remains compatible with existing SLTH constructions \citep{Pensia}.

Most existing theoretical results establishing the SLTH rely on \emph{unstructured pruning}, also known as \emph{weight pruning}, in which individual edges are removed from the network.
A key technical insight, first identified in \cite{malachProvingLotteryTicket2020} and later sharpened in \cite{Pensia}, is that weight pruning allows one to approximate target weights by doing subset sums of randomly initialized weights.
By leveraging a classical result of \cite{Lueker98} on the random subset sum problem, \cite{Pensia} showed that a random network with only logarithmic overparameterization—on the order of $O(\log(d/\varepsilon))$ per target weight—contains subnetworks that $\varepsilon$-approximate a broad family of target functions.
Matching lower bounds show that this logarithmic dependence is essentially optimal for constant-depth networks \citep{Pensia}.
While unstructured pruning provides a powerful theoretical framework for
establishing the existence of accurate subnetworks, its implications for
computational efficiency are less clear.
Indeed, unstructured pruning typically produces sparsity patterns with zeros in arbitrary locations in the weight matrices.
Such sparsity does not, by itself, translate into predictable speedups on
standard hardware, where dense linear algebra routines are optimized for
contiguous memory access and regular computation.

The aforementioned gap between parameter sparsity and actual speedups motivates the study of \emph{structured pruning} methods, which enforce sparsity at the level of entire blocks, rows, or columns of weight matrices.
Such structure directly translates into smaller matrices and fewer arithmetic operations, yielding genuine reductions in memory usage and inference time.
Among structured approaches, the simplest and most natural is \emph{neuron pruning}, in which entire hidden units are removed, effectively deleting rows (and corresponding columns) of the network’s weight matrices.
Understanding the expressive power and limitations of neuron pruning is therefore essential for bridging theoretical guarantees with practically meaningful efficiency gains.
In contrast to weight pruning, the role of neuron pruning has received comparatively little theoretical attention.
This is not accidental:
\cite{malachProvingLotteryTicket2020} explicitly observed that neuron pruning appears fundamentally weaker than weight pruning, and pointed to results on random feature models to support this claim.
In particular, \citet{yehudaiShamirRandomFeatures} showed that there exist target
ReLU neurons with bias for which approximation by random features is provably
hard: unless one allows either an exponential number of random features or
exponentially large coefficients, uniform approximation over standard Gaussian
inputs is impossible.
Subsequent work strengthened this result by showing with a similar proof strategy that even when the
magnitude of the coefficients is unrestricted, approximating such target ReLU
neurons still requires an exponential number of random features
\citep{DBLP:conf/colt/KamathMS20}.
These impossibility results are proved under the assumption that the target ReLU neuron may have a bias as large as $O(d^4)$, where $d$ is the dimension of the input.
It therefore leaves open the possibility that the observed hardness of approximation
is driven, at least in part, by this bias assumption rather than by an inherent
limitation of neuron-pruning\footnote{The proof in
    \cite{yehudaiShamirRandomFeatures} relies on the presence of such a bias, and
    removing this assumption would require a fundamentally different proof strategy, as confirmed
    by the authors via private communication.}.
This raises the following natural questions:
\emph{does the inefficiency of neuron pruning persist when one rules out biases? What is the dependence of neuron pruning w.r.t. $\eps$?}

\paragraph{Our Contribution}
In this work, we revisit neuron pruning in a clean and natural setting.
We focus on approximating a single ReLU neuron using a larger randomly initialized ReLU network with one hidden layer, where both the target neuron and the hidden neurons have zero bias.
This setting isolates the expressive limitations of neuron pruning from confounding effects caused by large biases.
Importantly, approximating a single neuron is arguably  the simplest nontrivial approximation task, so lower bounds in this regime already indicate fundamental limitations of neuron pruning.
We show that neuron pruning fundamentally requires a starting network with $\Omega(d/\eps)$ hidden neurons to successfully $\eps$-approximate a target ReLU neuron, when the output weights are bounded away from zero (\Cref{thm:pruning-lower-bound}). We further show that in the fully random initialization setting, where the output weights are also random, a weaker lower bound of $\Omega(\sqrt{d}\log (d)/\eps)$ still holds (\Cref{thm:pruning-lower-bound-log}). 
This establishes an exponential gap between weight and neuron pruning in their dependence on the approximation accuracy $\eps$,
as weight pruning achieves an $\varepsilon$-approximation with only
$O(d\log(d/\varepsilon))$ units \cite{Pensia}.
This suggests that neuron pruning is not a viable SLTH mechanism in deep networks: in existing constructions, the required approximation accuracy $\varepsilon$ must shrink with the target network’s width and depth, making a linear dependence on $1/\varepsilon$ prohibitive, whereas weight pruning only incurs logarithmic dependence on $1/\varepsilon$ (see \cite{Pensia}).

At a technical level, to handle the bias-free setting, we adopt a novel proof strategy in which we track the location of hidden units’ nonlinearities along carefully constructed input families and derive necessary conditions for approximation. We then reformulate these conditions as hitting events of suitable stochastic processes, whose hitting probabilities yield bounds on the success probability of approximation.

\section{Related Work}
\label{sec:related}
After the Lottery Ticket Hypothesis was introduced in \citet{frankleLotteryTicketHypothesis2018}, a large body of work has investigated algorithms for identifying such subnetworks, including pruning strategies that learn masks or importance scores for individual weights without changing their initial values \citep{zhouDeconstructingLotteryTickets2019,ramanujanWhatHiddenRandomly2020,WangZXZSZH20}.
These empirical results lead to the Strong Lottery Ticket Hypothesis (SLTH), which claims that pruning alone can reveal subnetworks that approximate some target network.
The first rigorous SLTH guarantees were obtained by \citet{malachProvingLotteryTicket2020}, who showed that sufficiently overparameterized random ReLU networks contain accurate subnetworks with high probability.
Subsequent works \citep{orseauLogarithmicPruningAll2020,Pensia} refined these guarantees by reducing the required overparameterization and extending the framework to broader architectural settings.
In particular, \citet{Pensia} established that logarithmic overparameterization in the approximation accuracy $\varepsilon$ is sufficient for dense networks through connections with the random subset sum problem.
Later works extended these results to convolutional and residual architectures \citep{dacunhaProvingStrongLottery2022,burkholzConvolutionalResidualNetworks2022,cunhaPolynomiallyOverParameterizedConvolutional2023} and to more general equivariant architectures \citep{ferbachGeneralFrameworkProving2022}.
Further refinements addressed broader activation functions, reduced depth overhead, and introduced notions such as universal lottery tickets \citep{burkholzMostActivationFunctions2022,burkholzExistenceUniversalLottery2022,fischerLotteryTicketsNonzero2022}.

Despite strong theoretical progress for weight pruning, unstructured sparsity does not directly translate into computational speedups on modern hardware.
This limitation has motivated extensive research on structured pruning methods, which remove entire blocks of parameters such as channels or neurons.
Early works related to structured pruning include classical studies on estimating the relevance of individual hidden units and removing those deemed unimportant \citep{MozerS88,MozerS89}.
Since then, structured pruning has developed into a broad research direction, particularly for convolutional architectures; we refer to the surveys in \citep{HoeflerABDP21,HeX23} for comprehensive overviews.
This line of work predominantly focuses on designing pruning algorithms and studying their empirical efficiency.
By contrast, structured pruning has received little attention within the SLTH literature.
Some recent works have established SLTH guarantees for some forms of structured convolutional pruning using multidimensional extensions of random subset sum techniques \citep{cunhaPolynomiallyOverParameterizedConvolutional2023}.
However, understanding the expressive power of neuron pruning, the simplest structured pruning strategy, remains largely open.

In particular, \citet{malachProvingLotteryTicket2020} observed that neuron pruning appears intrinsically weaker than weight pruning, by using previous results \citep{yehudaiShamirRandomFeatures} on random features model.
In a random features model, hidden weights are fixed and only output coefficients can be trained.
Indeed, if the output coefficients associated with retained neurons were allowed to be refitted after pruning, neuron pruning would reduce to selecting a subset of random features and learning their linear combination.
Neuron pruning in the SLTH setting is strictly more restrictive, as both the hidden weights and output coefficients are inherited from the original random network and cannot be changed.
Several works have established strong lower bounds for random feature models that highlight intrinsic approximation limitations.
In particular, \citet{yehudaiShamirRandomFeatures} showed that approximating certain ReLU neurons with bias using random features requires either exponentially many features or exponentially large coefficients under Gaussian inputs.
This result was strengthened in \citet{DBLP:conf/colt/KamathMS20}, which proved exponential lower bounds even when the magnitude of output coefficients is unrestricted.
More lower bounds for random feature models were also established in
\cite{hsuApproximationTwoLayer2021}, with a focus on approximating smooth functions.
However, the aforementioned results rely on target neurons with large biases, and it remains unclear to what extent such assumptions are necessary or whether similar barriers arise in bias-free settings.
These thus provide indirect evidence that neuron pruning may face fundamental limitations compared to weight pruning.

\section{Preliminaries and Setup}
\label{sec:preliminaries}
\begin{figure}[t]
\centering
\resizebox{\linewidth}{!}{ %
\begin{tikzpicture}[
  font=\small,
  box/.style={draw, rounded corners=2mm, inner sep=3mm, align=center},
  neuron/.style={circle, draw, minimum size=4.2mm, inner sep=0pt},
  conn/.style={-{Latex[scale=0.7]}, line width=0.6pt},
  pruned/.style={opacity=0.25},
  prunedx/.style={red!70, line width=0.8pt},
  note/.style={align=left},
  scale=1
]

\node[box] (tbox) at (0,0) {%
  \textbf{Target neuron}\\[-0.5mm]
  $f(\mathbf{x})=\sigma(\langle \mathbf{w}^\star,\mathbf{x}\rangle)$
};

\node[box] (ubox) at (5.2,0) {%
  \textbf{Unpruned network}\\[-0.5mm]
  $g(\mathbf{x})=\sum_{i=1}^{N_h}\alpha_i\,\sigma(\langle \mathbf{w_i},\mathbf{x}\rangle)$
};

\node[box] (pbox) at (11,0) {%
  \textbf{Pruned network}\\[-0.5mm]
  $g_S(\mathbf{x})=\sum_{i\in S}\alpha_i\,\sigma(\langle \mathbf{w_i},\mathbf{x}\rangle)$
};

\node[neuron] (t_x1) at ($(tbox.south)+(-1.25,-1.15)$) {};
\node[neuron] (t_x2) at ($(tbox.south)+(-1.25,-1.85)$) {};
\node[neuron] (t_out) at ($(tbox.south)+(1.20,-1.50)$) {};

\node at ($(t_x1)+(0,-0.0)$) {$x_1$};
\node at ($(t_x2)+(0,-0.0)$) {$x_2$};
\node at ($(t_out)+(0,-0.5)$) {$f(\mathbf{x})$};
\node at (t_out) {$\sigma$};

\draw[conn] (t_x1) -- node[pos=0.5, above, yshift=-1pt] {$w_1^\star$} (t_out);
\draw[conn] (t_x2) -- node[pos=0.5, below] {$w_2^\star$} (t_out);

\node[neuron] (u_x1) at ($(ubox.south)+(-2.25,-1.15)$) {};
\node[neuron] (u_x2) at ($(ubox.south)+(-2.25,-1.85)$) {};
\node[neuron] (u_out) at ($(ubox.south)+(2.10,-1.50)$) {};

\node at ($(u_x1)+(0,-0.00)$) {$x_1$};
\node at ($(u_x2)+(0,-0.00)$) {$x_2$};
\node at ($(u_out)+(0,-0.5)$) {$g(\mathbf{x})$};

\foreach \i in {1,...,5} {
  \node[neuron] (uh\i) at ($(ubox.south)+(-0.10,{ -1.50 + (3-\i)*0.58})$) {};
  \node at (uh\i) {\tiny $\sigma$};

  \draw[conn] (u_x1) -- (uh\i);
  \draw[conn] (u_x2) -- (uh\i);
  \draw[conn] (uh\i) -- (u_out);

}

\coordinate (u_alphacol) at ($(u_out)+(-1.3,0.95)$);
\node[anchor=east] at ($(u_alphacol)+(0,{(0.8-0.6)*0.5})$) {\scriptsize $\alpha_{1}$};
\node[anchor=east] at ($(u_alphacol)+(0,{(0.8-1.5)*0.5})$) {\scriptsize $\alpha_{2}$};
\node[anchor=east] at ($(u_alphacol)+(0,{(0.8-2.4)*0.5})$) {\scriptsize $\alpha_{3}$};
\node[anchor=east] at ($(u_alphacol)+(0,{(0.8-3.2)*0.5})$) {\scriptsize $\alpha_{4}$};
\node[anchor=east] at ($(u_alphacol)+(0,{(0.8-4)*0.5})$) {\scriptsize $\alpha_{5}$};

\draw[-Latex, line width=0.8pt] ($(ubox.east)+(0.2,-0.1)$) -- ($(pbox.west)+(-0.2,-0.1)$);

\node[neuron] (p_x1) at ($(pbox.south)+(-2.25,-1.15)$) {};
\node[neuron] (p_x2) at ($(pbox.south)+(-2.25,-1.85)$) {};
\node[neuron] (p_out) at ($(pbox.south)+(2.1,-1.50)$) {};

\node at ($(p_x1)+(0,-0.0)$) {$x_1$};
\node at ($(p_x2)+(0,-0.0)$) {$x_2$};
\node at ($(p_out)+(0,-0.5)$) {$g_S(\mathbf{x})$};

\foreach \i in {1,...,5} {
  \node[neuron] (ph\i) at ($(pbox.south)+(-0.10,{ -1.50 + (3-\i)*0.58})$) {};
  \node at (ph\i) {\tiny $\sigma$};

  \draw[conn] (p_x1) -- (ph\i);
  \draw[conn] (p_x2) -- (ph\i);
  \draw[conn] (ph\i) -- (p_out);

}

\coordinate (p_alphacol) at ($(p_out)+(-1.3,0.95)$);
\node[anchor=east] at ($(p_alphacol)+(0,{(0.8-0.6)*0.5})$) {\scriptsize $\alpha_{1}$};
\node[anchor=east] at ($(p_alphacol)+(0,{(0.8-1.5)*0.5})$) {\scriptsize $\alpha_{2}$};
\node[anchor=east] at ($(p_alphacol)+(0,{(0.8-2.4)*0.5})$) {\scriptsize $\alpha_{3}$};
\node[anchor=east] at ($(p_alphacol)+(0,{(0.8-3.2)*0.5})$) {\scriptsize $\alpha_{4}$};
\node[anchor=east] at ($(p_alphacol)+(0,{(0.8-4)*0.5})$) {\scriptsize $\alpha_{5}$};

\foreach \i in {2,4} {
  \begin{scope}[pruned]
    \fill[white] (ph\i) circle (2.2mm);
    \draw[neuron] (ph\i) circle (2.15mm);
    \draw[conn] (p_x1) -- (ph\i);
    \draw[conn] (p_x2) -- (ph\i);
    \draw[conn] (ph\i) -- (p_out);
  \end{scope}
  \draw[prunedx] ($(ph\i)+(-2.0mm,-2.0mm)$) -- ($(ph\i)+(2.0mm,2.0mm)$);
  \draw[prunedx] ($(ph\i)+(-2.0mm,2.0mm)$) -- ($(ph\i)+(2.0mm,-2.0mm)$);
}

\end{tikzpicture}
}
\caption{Target ReLU neuron $f$, random network $g$ ($N_h=5$), and subnetwork $g_S$ pruned from $g$ by retaining a subset $S$ of hidden units. For input dimension $d=2$, each unit $i$ has incoming weights $\mathbf{w}_{i} \in \mathbb{R}^2$ (not shown) and output weight $\alpha_i \in \mathbb{R}$.}\label{fig:overview}
\label{fig:overview}
\end{figure}

In this work, vectors are denoted in boldface (e.g., $\mathbf{x}, \mathbf{w}$), while scalars are written in regular font.
For vectors $\mathbf{u}, \mathbf{v} \in \mathbb{R}^d$, we write $\langle \mathbf{u}, \mathbf{v} \rangle$ for their dot product.
We use $\sigma(t) := \max\{t,0\}$ to denote the ReLU activation function.
When we say that a function $h : \mathbb{R} \to \mathbb{R}$ is \emph{linear} on an interval, we mean that it coincides on that interval with an \emph{affine} function of the form $t \mapsto at + b$, for some $a,b \in \mathbb{R}$.

In \Cref{sec:main}, we study the problem of approximating a single target ReLU neuron
using a larger randomly initialized neural network $g$, solely by pruning neurons in the hidden layer.
Specifically, we consider a bias-free two-layer ReLU network of the form
$
    g(\mathbf{x})
    =
    \sum_{i=1}^{N_h} \alpha_i\,\sigma(\langle \mathbf{w}_i, \mathbf{x}\rangle),
$
where $\mathbf{x}\in\mathbb{R}^d$, with hidden-layer neuron weights $\mathbf{w}_i\in\mathbb{R}^d$ and output weights
$\alpha_i\in\mathbb{R}$.
\Cref{fig:overview} provides a visual overview of the neuron pruning setting and the notation we use.

The target function we will consider in our work is a single bias-free ReLU neuron
$
    f(\mathbf{x}) := \sigma(\langle \mathbf{w}^\star, \mathbf{x}\rangle)
$, with $\|\mathbf{w}^\star\|_2 = \sqrt{d}$.
A \emph{neuron-pruned subnetwork} $g_S$ is obtained from $g$ by selecting a subset
$S \subseteq \{1,\dots,N_h\}$ of hidden units and retaining exactly those neurons with all their incident edges, yielding
$
    g_S(\mathbf{x})
    =
    \sum_{i\in S} \alpha_i\,\sigma(\langle \mathbf{w}_i, \mathbf{x}\rangle).
$ %
In line with the Strong Lottery Ticket Hypothesis, pruning is the only operation allowed:
the retained weights are not trained or modified.

Throughout the paper, we will often use the following notion of $\varepsilon$-approximation.
\begin{definition}[$\varepsilon$-approximation]
    \label{def:eps-apx}
    For $\varepsilon>0$ and radius $R\ge 1$, we say that a function $h$ \emph{$\varepsilon$-approximates} another function $f$ on the ball of radius $R$ if
    $
        \sup_{\|x\|_2\le R} |h(x)-f(x)| \le \varepsilon.
    $
\end{definition}

\section{Main Results}
\label{sec:main}
\Cref{thm:pruning-lower-bound} shows that, in a one hidden layer random network with output weight bounded away from zero, neuron pruning alone
cannot even $\varepsilon$-approximate a single target ReLU neuron without bias, unless the initial width
is $\Omega(d/\varepsilon)$.
This section presents a sketch of the proof of the theorem, in which we introduce the main ideas and provide an overview of the argument, referring the reader to \Cref{sec:main-proof} for the complete, detailed proof.

For simplicity, in \Cref{thm:pruning-lower-bound} we fix the output weights of network $g$ to be equal to $1$. The same proof strategy extends to output-weight distributions that are uniformly bounded away from zero, such as Rademacher; see \Cref{sec:appendix-output-weights} for the corresponding adaptation. 
Also, we require input dimension $d \ge 2$; for $d=1$, neuron pruning reduces to weight pruning, and we can directly use the pruning lower bound proved in \cite{Pensia}.
The requirement that $\|\mathbf w^\star\|_2=\sqrt{d}$ is mainly to simplify the notation in the proof; if, instead, we assume $\|\mathbf w^\star\|_2=1$, then with straightforward modifications one obtains the lower bound $ N_h \;\geq\; \min\{c_0\frac{\sqrt d}{\varepsilon},\,2^{c_1 d}\}$.

\begin{theorem}
\label{thm:pruning-lower-bound}
Let $d \ge 2$, $\varepsilon \in (0,1)$.
Consider a one hidden-layer ReLU network without bias and unit output weights of the form
$
    g(\mathbf{x}) \;=\; \sum_{i=1}^{N_h} 1 \cdot \sigma(\langle \mathbf{w_i}, \mathbf{x}\rangle),
$
where the weights $\{\mathbf{w_i}\}_{i=1}^{N_h}$ are drawn independently from $\mathcal{N}(0,I_d)$, and $\mathbf{x} \in \mathbb{R}^d$.
Then, there exist a target neuron with weight vector $\mathbf w^\star\in\mathbb R^d$ with $\|\mathbf w^\star\|_2=\sqrt{d}$, and  universal constants $c_0,c_1>0$ such that, if
\[
    \textstyle
    N_h \;<\; \min\left\{c_0\frac{d}{\varepsilon},\,2^{c_1 d}\right\},
\]
then with probability at least $1-e^{-\Omega(d)}$ over the draw of $\{\mathbf{w_i}\}_{i=1}^{N_h}$, for every subset $S \subseteq \{1,\dots,N_h\}$ and every fixed constant $R>2$,
\[
    \textstyle
    \sup_{\|\mathbf{x}\|_2 \le R}
    \left|
        \sum_{i\in S} \sigma(\langle \mathbf{w_i},\mathbf{x}\rangle)
        \;-\;
        \sigma(\langle \mathbf{w^\star},\mathbf{x}\rangle)
    \right|
    \;\ge\;
    C\varepsilon,
\]
for some universal constant $C>0$.
\end{theorem}

\begin{proofsketch}
    As neuron pruning may retain any number $k \in \{1, \dots, N_h \}$ of hidden units, we begin in \Cref{sec:union-bound} by controlling the probability of successful approximation uniformly over all possible pruned network sizes. This is achieved via a union bound over $k$, which reduces the analysis to bounding, for a fixed $k$, the probability $p_k$ that a random network with $k$ hidden units $\varepsilon$-approximates the target.

    Our analysis is then based on tracking the evolution of \emph{breakpoints}—the locations at which ReLU activations of the hidden neurons change slope—along carefully chosen one-dimensional input families, as described in \Cref{sec:restriction-input}. When a high-dimensional ReLU network is restricted to such paths, both the target function and the approximating network reduce to piecewise-linear functions whose behavior is entirely determined by the number, location, and interaction of these breakpoints.
    Crucially, these breakpoints impose  necessary conditions for approximation.
    \Cref{def:broken-bin} formalizes when a breakpoint induces a non-negligible approximation
    error.
    Lemma~\ref{lem:broken-bin-prevents-approx} shows that the presence of a breakpoint away
    from the target nonlinearity prevents $\varepsilon$-approximation, while
    Lemma~\ref{lem:breakpoint-necessary} establishes that at least one breakpoint must be
    placed within an $\varepsilon$-neighborhood of the target nonlinearity.
    Together, these results imply that successful neuron pruning requires introducing exactly
    one suitably aligned breakpoint and canceling all others.

    We formalize this constraint by viewing neuron pruning as a stochastic sequential process
    in which neurons are selected one by one and the number of unresolved breakpoints evolves
    over time.
    The resulting process, faithfully captures the breakpoint
    dynamics induced by pruning.
    To make this process tractable, we construct a sequence of couplings:
    first to a capped process that limits the total number of breakpoints,
    and then to a homogeneous birth--death process 
    (\Cref{sec:capped-process,sec:birth-death}).
    Each coupling is designed so that the dominating process is strictly more favorable to
    successful approximation, and therefore yields an upper bound on the success probability
    of the original pruning process.
    
    The above argument bounds allows us to derive an upper bound for the $\varepsilon$-approximation of a single one-dimensional input family (\Cref{lem:single-family-ub}).
    To obtain dimension-dependent bounds, we exploit the fact that our construction uses
    $\lfloor d/2 \rfloor$ input families supported on disjoint coordinate pairs.
    The induced breakpoint processes are independent across these families, and therefore
    the overall success probability decays exponentially in the input dimension $d$ (\Cref{lem:all-families-ub}); see also \Cref{fig:RWs} in \Cref{apx:figures} for a visual intuition.
    
    Finally, \Cref{sec:final-union-bound} combines these bounds on $p_k$ with the initial union bound
    over all pruned subnetworks, yielding the stated lower bound on the number of hidden neurons required for
    neuron pruning to achieve $\varepsilon$-approximation with non-negligible probability.
\end{proofsketch}
We also consider a variant of \Cref{thm:pruning-lower-bound} in which the output weights are randomly initialized. In this setting, we obtain a lower bound on the required width $N_h$ with the same dependence on $1/\varepsilon$, while the scaling in the input dimension is weaker.
The proof strategy is substantially different from that of \Cref{thm:pruning-lower-bound}, and it relies on a simpler necessary condition for approximation.
The proof is provided in \Cref{apx:proof-lower-bound-log}. 
In \Cref{thm:pruning-lower-bound-log}, if we instead consider a target weight vector with unit norm $\|\mathbf{w}^\star\|_2=1$, the lower bound on $N_h$ loses a factor of $\sqrt{d}$, scaling instead as $\nicefrac{\log d}{\varepsilon}$.

\begin{restatable}{theorem}{pruninglowerboundlog}
\label{thm:pruning-lower-bound-log}
Let $d \ge 2$, $R>2$, $\eps \in (0,1)$, $\delta\in(0,1)$, and let $\mathbf w^\star\in\mathbb R^d$ with $\|\mathbf w^\star\|_2=\sqrt{d}$.
Consider
$
g(\mathbf x)=\sum_{j=1}^{N_h}\alpha_j\,\sigma(\langle \mathbf w_j,\mathbf x\rangle),
$
where $\mathbf w_j\sim\mathcal N(0,I_d)$ and $\alpha_j\sim\mathcal N(0,1)$ are independent.
Then there exist a constant $c_0 = c_0(\delta)>0$ such that, if
\[
\textstyle
N_h < c_0\frac{\sqrt{d} \log d}{\varepsilon},
\]
then with probability at least $1-\delta$ over the initialization, for every $S\subseteq\{1,\dots,N_h\}$,
$
\sup_{\|\mathbf x\|_2\le R}
\left|
\sum_{j\in S}\alpha_j\,\sigma(\langle \mathbf w_j,\mathbf x\rangle)
-\sigma(\langle \mathbf w^\star,\mathbf x\rangle)
\right|
\ge \varepsilon / 4.
$
\end{restatable}

\section{Proof of \Cref{thm:pruning-lower-bound}}
\label{sec:main-proof}
In this section, we prove \Cref{thm:pruning-lower-bound}. Because of space constraints, some proofs of intermediate results are omitted and given in \Cref{apx:omitted-proofs}.

\subsection{Union bound over all pruned subnetworks}
\label{sec:union-bound}
For any $k \in \{1,\dots,N_h\}$, pruning $g$ down to $k$ hidden neurons
produces $\binom{N_h}{k}$ distinct subnetworks.
Let $E_k$ be the event that there is a $k$-neuron
pruned subnetwork that $\varepsilon$-approximates the target, and let
$p_k := \Pr(\text{a fixed $k$-neuron network $\varepsilon$-approximates the target})$.
Then, by a union bound,
\begin{equation}
\textstyle
\Pr\Bigl(\bigcup_{k=1}^{N_h} E_k\Bigr)
\;\le\;
\sum_{k=1}^{N_h} \binom{N_h}{k}\, p_k .
\label{eq:union-bound}
\end{equation}
The rest of the proof is devoted to quantifying the probability $p_k$.

\subsection{Restriction over simple input families}
\label{sec:restriction-input}
We introduce some simple families of inputs that will be used to derive necessary
conditions on the function $g$ for neuron pruning to succeed.
For each $i \in \{1, \dots, \lfloor d/2 \rfloor\}$, define the input family $\mathbf{x}_i(t) := t \mathbf{e}_{2i-1} + \mathbf{e}_{2i}$ for $t \in \mathbb{R}$, where $\mathbf{e}_j$ denotes the $j$-th standard basis vector. This family has exactly two nonzero coordinates: a variable entry $t$ at index $2i-1$ and a fixed entry $1$ at index $2i$.
There are $\lfloor d/2 \rfloor$ such families, supported on disjoint pairs of
coordinates.

Recall that $\mathbf{w^\star} = (w_1^\star,\dots,w_d^\star) \in \mathbb{R}^d$ denotes the target weight vector.
We choose $\mathbf{w^\star} = (1,\ldots,1)^\top$.
Along the input path $\mathbf{x}_i(t)$, the output of the target neuron is
\[
    f(\mathbf{x}_i(t))
    =
    \sigma(\langle \mathbf{w^\star}, \mathbf{x}_i(t)\rangle)
    =
    \sigma\bigl(w_{2i-1}^\star\, t + w_{2i}^\star\bigr),
\]
which defines a one-dimensional piecewise-linear function of $t$ with a single
\emph{breakpoint}
$
    t_i^\star
    :=
    -\nicefrac{w_{2i}^\star}{w_{2i-1}^\star}
    = -1.
$
On one side of $t_i^\star$ the function is affine with slope $w_{2i-1}^\star=1$, and on the other side it is identically zero.
Similarly, for the randomly initialized network $g(x)$, for each hidden neuron
$j$ with an incoming weight vector $\mathbf{w_j}=(w_{j,1},\dots,w_{j,d})$, and for each input family
$\mathbf{x}_i(t)$, we have
\[
    \sigma(\langle \mathbf{w_j}, \mathbf{x}_i(t)\rangle)
    =
    \sigma\bigl(w_{j,\,2i-1}\, t + w_{j,\,2i}\bigr).
\]
Each hidden neuron $j$ therefore introduces its own breakpoint
$
    t_{i,j}
    \;:=\;
    -\nicefrac{w_{j,\,2i}}{w_{j,\,2i-1}}
$
for the input family $\mathbf{x}_i(t)$.
Consequently, given a subset $S \subseteq \{1,\dots,N_h\}$ of hidden neurons, the restriction $g_S(\mathbf{x}_i(t))$ of the pruned network $g_S$ on the input family $\mathbf{x}_i(t)$ yields a piecewise-linear function
of $t$, whose breakpoints are induced by the neurons in $S$.

\subsection{Breakpoints and necessary conditions for approximation}
\label{sec:breakpoints-original}
To analyze the probability $p_k$, it is convenient to view neuron pruning as a
sequential selection process in which the $k$ neurons that are kept in the final subnetwork are chosen one at a time.

Fix an index $i \in \{1,\dots,\lfloor d/2 \rfloor\}$ and consider the restriction of the network
to the input family $\mathbf{x}_i(t)$ introduced above, in which each hidden neuron induces a single breakpoint at a location
determined by its weights.
As neurons are selected sequentially, each newly added neuron can have one of the
following effects on the breakpoint structure of the current partial sum of chosen
neurons for the input family $\mathbf{x}_i(t)$: 
(i) it may introduce a new breakpoint,
(ii) it may cancel an existing breakpoint through interaction with previously
selected neurons, or
(iii) it may leave the breakpoint structure unchanged.
After all $k$ neurons have been selected, the restriction of the resulting
pruned network to $\mathbf{x}_i(t)$ is a piecewise-linear function whose breakpoints are
entirely determined by the selected neurons.

Since we are interested in inputs $x$ satisfying $\|x\|_2 \le R$, for an input
family $\mathbf{x}_i(t)$ this restriction implies $|t| \le \sqrt{R^2 - 1}$.
Consequently, only breakpoints lying in the interval $I_R = [-\sqrt{R^2 - 1},\, \sqrt{R^2 - 1}]$ can affect the approximation error on bounded inputs.
Breakpoints outside this interval have no effect and will therefore be ignored and not counted.

We partition $I_R$ into $\nicefrac R\varepsilon$ subintervals (bins) $I_{B_1},...,I_{B_{\nicefrac R\varepsilon}}$ of length $\varepsilon$, and denote this
partition by $\mathcal{P}_\varepsilon$.
Given a subset $S \subseteq \{1,\dots,N_h\}$ of hidden neurons, we say that a generic bin $I_{B} \in \mathcal{P}_\varepsilon$ is \emph{broken} for $g_S$ and input family $\mathbf{x}_i(t)$ if $g_S(\mathbf{x}_i(t))$ is non-linear in $I_B$ (see \Cref{def:broken-bin} for a formal statement).
Bins that contain only canceled breakpoints, or no breakpoints at all, are considered
\emph{unbroken}.

\begin{definition}[Broken bin]
    \label{def:broken-bin}
    Let $I_B \subset \mathbb{R}$ be an interval (bin) of length $\varepsilon$, and let
    $h : I_B \to \mathbb{R}$ be a function.
    We say that $I_B$ is \emph{broken} for $h$ if there exist three points
    $t_1 < t_2 < t_3$ in $I_B$ such that for every affine function $\ell(t)=at+b$,
    $
        \max_{j=1,2,3} |h(t_j)-\ell(t_j)|
        \;\ge\;
        c\,\varepsilon
    $
    for some universal constant $c>0$.
\end{definition}
By \Cref{lem:broken-bin-prevents-approx} and \Cref{lem:breakpoint-necessary}, $\varepsilon$-approximation along the path $\mathbf{x}_i(t)$ requires the breakpoint structure of the pruned network to be aligned with that of the target: 
after selecting $k$ neurons, the only possible broken bin is the one containing the target breakpoint $t_i^\star$, and no additional broken bins may occur.
\begin{restatable}[Broken bin prevents approximation]{lemma}{brokenbinpreventsapprox}
    \label{lem:broken-bin-prevents-approx}
    Let
    $
        f(t) := \sigma(a t + b),
        $ with $a \neq 0,
    $
    and let $I_B \subset \mathbb{R}$ be an interval of length $\varepsilon$ such that
    the breakpoint
    $t^\star := -\nicefrac{b}{a}$
    satisfies $t^\star \notin I_B$.
    Let $h : \mathbb{R} \to \mathbb{R}$ be any function.
    If $I_B$ is broken for $h$ in the sense of \Cref{def:broken-bin}, then
    $
        \sup_{t \in B} |f(t) - h(t)|
        \;\ge\;
        c\,\varepsilon,
    $
    for a universal constant $c>0$.
\end{restatable}
\begin{restatable}[A breakpoint is necessary for approximation]{lemma}{breakpointnecessary}
    \label{lem:breakpoint-necessary}
    Let
    $
        f(t) := \sigma(a t + b),
        $ $a \neq 0,$
    and let
    $t^\star := -\nicefrac{b}{a}$
    be its breakpoint.
    Let $h : \mathbb{R} \to \mathbb{R}$ be a function that is linear on an interval
    $I$ containing $[t^\star-\delta,\, t^\star+\delta]$ for some $\delta>0$.
    Then,
    $
        \sup_{t \in [t^\star-\delta,\, t^\star+\delta]}
        |f(t) - h(t)|
        \;\ge\;
        |a|\,\delta /4.
    $
\end{restatable}
\noindent It is convenient to rephrase the necessary $\varepsilon$-approximation conditions derived from \Cref{lem:broken-bin-prevents-approx} and \Cref{lem:breakpoint-necessary} in an alternative form that simplifies
the following analysis. Rather than viewing the goal as approximating the target ReLU
itself, we conceptually start from a mirrored version of the target along
$\mathbf{x}_i(t)$, obtained by flipping its output sign while preserving the same
breakpoint location $t_i^\star$. This mirrored function induces exactly one broken
bin, namely the bin containing $t_i^\star$. A neuron whose breakpoint lies in this
same bin and whose slope jump is suitably close to the target's slope will cancel the mirrored breakpoint,
so that their sum is linear on the bin and hence the bin becomes unbroken.

From this viewpoint, successful $\varepsilon$-approximation on the input family $\mathbf{x}_i(t)$ implies
starting from a configuration with a single broken bin (in which the target breakpoint lies) and requiring that, after
selecting $k$ neurons, no broken bins remain. We adopt this reformulation
throughout the remainder of the proof: the bin containing $t_i^\star$ is treated
as initially broken, and we consider approximation to be successful only if after selecting $k$ neurons we have no broken bins.

Using this reformulation, we now describe the sequential selection of $k$ neurons as a stochastic process that tracks the number of broken bins.
As we will introduce other stochastic processes later in the proof, we will refer to this one as the \emph{original pruning process}.
For each step $s = \{1,\dots,k\}$, a new neuron is added and it may increase by one, decrease by one, or leave unchanged the number of broken bins in $\mathcal{P}_\varepsilon$, depending of where its breakpoint falls. 
Since we treat the bin containing $t_i^\star$ as initially broken, the starting number of broken bins at step $s=0$ will be one.
\begin{definition}
    We define $B^{\mathrm{orig}}_s$ as the number of broken bins after having selected $s$ neurons in the original pruning process, with $B^{\mathrm{orig}}_0=1$. 
\end{definition}
As noted above, for a successful $\varepsilon$-approximation on the input family $\mathbf{x}_i(t)$, we want to have $B_k^{\mathrm{orig}}=0$.
Now consider all $\lfloor d/2 \rfloor$ disjoint input families $\mathbf{x}_i(t)$. 
For the overall approximation to be successful, we need to have $B_k^{\mathrm{orig}}=0$ for all such families.
By construction, the breakpoints $t_{i,j}$ are independent across
all $\lfloor d/2 \rfloor$ families $\mathbf{x}_i(t)$, so 
$ p_k \le
    \prod_{i=1}^{\lfloor d/2 \rfloor}
    \Pr(B_k^{\mathrm{orig}} = 0)
$, which gives the following.
\begin{lemma}
\label{fact:orig-success-prob}
    It holds $p_k \le (\Pr(B_k^{\mathrm{orig}} = 0))^{\lfloor d/2 \rfloor}$.
\end{lemma}

\subsection{Construction of a dominating capped process}
\label{sec:capped-process}
In this section, we introduce a simplified process that
stochastically dominates the original pruning process described in \Cref{sec:breakpoints-original}, in the sense that it produces fewer broken
bins.
Analyzing this simplified process will yield an upper bound on the probability that the
original pruning procedure succeeds.
As in the previous section, fix an index $i \in \{2,\dots,d\}$ and consider the restriction of the network
to the input family $\mathbf{x}_i(t)$.

We define a \emph{capped process}, analogous to the original pruning  process defined in \Cref{sec:breakpoints-original}, except for the fact that the number of broken bins is constrained not to exceed a fixed threshold. Fix an integer $T \in \{1,\dots,k\}$.
The number $B_s^{\mathrm{cap}}$ of broken bins in the capped process after selecting $s$ neurons, for $s \in \{1,\dots,k\}$, is updated according to these rules:
\begin{itemize}[noitemsep,nolistsep,leftmargin=2em]
    \item If $B_s^{\mathrm{cap}} < T$, as for the original process,  $B_s^{\mathrm{cap}}$ may increase by one if the newly selected neuron
          introduces a breakpoint that breaks a previously unbroken bin; it may decrease
          by one if the neuron introduces a breakpoint that cancels an existing one; otherwise,
          $B_s^{\mathrm{cap}}$ remains unchanged.
    \item If $B_s^{\mathrm{cap}} = T$, then even if breakpoint falls into a previously unbroken bin, we do not increase $B_s^{\mathrm{cap}}$ further; $B_s^{\mathrm{cap}}$ may still decrease if cancellations happen.
\end{itemize}

By construction, after having selected $s$ shared neurons, if we look at the capped and original processes we have 
\begin{equation}
    \label{eq:capped-process-dominates}
    B_s^{\mathrm{cap}} \leq B_s^{\mathrm{orig}}, \qquad \forall s \in \{1,\dots,k\}
\end{equation}
The capped process is more favorable to maintaining a small number of broken bins.
Consequently, the probability that the original pruning process reaches 0 broken bins after $k$ steps is upper bounded by the corresponding probability for the
capped process. In particular, from \cref{eq:capped-process-dominates}, we have the following upper bound on the probability of a successful $\varepsilon$-approximation for a random network with $k$ hidden neurons.
\begin{lemma}
\label{fact:cap-succ-prob}
    It holds $\Pr(B_k^{\mathrm{orig}} = 0) \le \Pr(B_k^{\mathrm{cap}} = 0)$.
\end{lemma}

\subsection{Construction of a dominating birth-death process}
\label{sec:birth-death}
In the capped process described in \Cref{sec:capped-process}, the probability that a selected neuron creates a new broken bin depends on the current number and locations of broken bins.
For example, this probability is high when few bins are broken and lower when most are already broken.
This dependence makes the exact process difficult to analyze.
To simplify the analysis, we introduce a
homogeneous birth-death process that we will use to upper bound $\Pr(B_k^{\mathrm{cap}} = 0)$.
\begin{definition}
    A $(q,p,T)$-chain ${(B_s^{\mathrm{bd}})}_{s \geq 0}$ is a birth-and-death process on a state space \\$\{0,1,\dots,T\}$ such that $B_0^{\mathrm{bd}}=1$ and
    \begin{itemize}[noitemsep,nolistsep,leftmargin=2em]
        \item from any state $b<T$, it transitions from state $b$ to $b+1$ with probability $p$;
        \item from any state $b>0$, it transitions from state $b$ to $b-1$ with probability $q$;
        \item otherwise, it remains at $b$.
    \end{itemize}
\end{definition}
Similarly to what we did in section \Cref{sec:capped-process}, where we introduced a capped process that keeps a lower number of broken bins than the original pruning process, here we want to choose a triple $(q,p,T)$ such that the associated $(q,p,T)$-chain ${(B_s^{\mathrm{bd}})}_{s \geq 0}$ stochastically dominates the capped process, in the sense that, for a suitable coupling, we have
\begin{equation}
\label{eq:birth-death-dominates}
B_s^{\mathrm{bd}} \leq B_s^{\mathrm{cap}}, \qquad \forall s \in \{1,\dots,k\}
\end{equation}
To achieve this, we can choose $p$ as a lower bound on the probability of introducing a new broken bin in the capped process, and $q$ as an upper bound on the probability
of removing a broken bin, both uniformly over all steps $s \in \{1,\dots,k\}$.
Such a choice ensures that \cref{eq:birth-death-dominates} holds and we get the following.
\begin{lemma}
\label{fact:bd-succ-prob}
    It holds $ \Pr(B_k^{\mathrm{cap}} = 0) \le \Pr(B_k^{\mathrm{bd}} = 0)$.
\end{lemma}
As far as the choice of the maximum number $T$ of broken bins goes, what we want is to keep it somehow smaller than the total
number of bins, in such a way that in the capped process the probability of introducing a broken bin is always more likely than removing one.
Let us fix, for a sufficiently large constant $c>0$,
\begin{equation}
\label{eq:cap-choice}
    \textstyle
    T = \frac{R}{c\varepsilon}.
\end{equation}

We now make more precise the choice of the probability $p$. Since the total number of bins is $\Theta(R/\varepsilon)$, as long as $b<T$, the fraction of broken bins is at most $1/c$, and the fraction of unbroken bins is at least $1-1/c$. By \Cref{lem:single-breakpoint-breaks-bin}, there exists a constant $\gamma>0$ such that a bin becomes broken in the sense of \Cref{def:broken-bin} whenever it contains a single well-placed breakpoint whose slope jump satisfies $|\Delta|\ge\gamma$. Here, “well-placed” means that the breakpoint lies away from the bin boundaries, as specified in \Cref{lem:single-breakpoint-breaks-bin}. 

For a given family, the breakpoint location of a neuron is a ratio of two independent standard Gaussians, while its slope jump is given by one of the Gaussian weights. As a result, the breakpoint location and slope jump have a joint density that is continuous and positive on the relevant range. Since a constant fraction of the bins are unbroken and a constant fraction of each such bin consists of well-placed locations, there exists a constant $p_0>0$ such that, uniformly over all steps, a newly sampled neuron lands a breakpoint in a suitable location of an unbroken bin and has slope jump at least $\gamma$ with probability at least $p_0$, which by \Cref{lem:single-breakpoint-breaks-bin} causes that bin to become broken. Therefore, we may set $p=p_0$.

\begin{restatable}[A single sufficiently steep neuron breaks a bin]{lemma}{singlebreakpointbreaksbin}
    \label{lem:single-breakpoint-breaks-bin}
    There exist universal constants $\gamma,c>0$ such that the following holds.
    Let $B=[u,u+\varepsilon]$ be an interval, and let $h:B\to\mathbb{R}$ be a
    piecewise-linear function with exactly one breakpoint $t_0\in [u+\eps/4,u+3\eps/4]$.
    Assume that the left and right slopes at $t_0$ satisfy
    $
        |m_+ - m_-| \ge \gamma .
    $
    Then, $B$ is broken for $h$ in the sense of \Cref{def:broken-bin}.
\end{restatable}
\noindent On the other hand, the fraction of broken bins is at most $1/c$.
Even assuming that cancellation always succeeds when a breakpoint falls inside a broken bin (i.e., the bin becomes unbroken), choosing $c$ in \cref{eq:cap-choice} sufficiently large ensures the existence of a constant $q_0<\min\{1/3, p_0/2\}$ such that, with probability at most $q_0$, uniformly over all steps, a newly added neuron changes a bin from broken to unbroken. We can then set $q = q_0$.

The resulting birth--death process stochastically dominates the capped process
in favor of successful approximation, and thus provides an upper bound on the probability
that neuron pruning yields a network $g_S$ that $\varepsilon$-approximates the target
ReLU.

We can now finally bound the probability $\Pr(B_k^{\mathrm{bd}} = 0)$ of such a $(p,q,T)$-chain to be in state 0 after $k$ steps, which by \Cref{fact:bd-succ-prob} gives us an upper bound on the probability of a successful $\varepsilon$-approximation.

A crucial part of the proof is now to show the following, which we prove via an elegant use of total probabilities. Because of space constraints, the proof is deferred to the appendix. 
\begin{restatable}{lemma}{singlefamilyub}
    \label{lem:single-family-ub}
    It holds 
    $
        \Pr(B_k^{\mathrm{bd}} =0 ) \;\le\; e^{-\Omega(\min(k,T))}.
    $
\end{restatable}
\noindent With \Cref{lem:single-family-ub}, we now have all the ingredients to establish an exponential upper bound for $p_k$.
Combining \Cref{fact:orig-success-prob,fact:cap-succ-prob,fact:bd-succ-prob,lem:single-family-ub} yields the following. 
\begin{lemma}
\label{lem:all-families-ub}
    It holds 
    $
    \label{eq:all-families-ub}
        p_k \le e^{-\Omega(\min(k,T) d)}.
    $
\end{lemma}
\Cref{lem:all-families-ub} gives the desired exponentially small upper bound on the probability of successful $\varepsilon$-approximation, as shown in the next section.

\subsection{Back to the union bound}
\label{sec:final-union-bound}

We now conclude the proof by returning to the union bound of \cref{eq:union-bound}.
We split the sum in two, depending on the value of $k$:
\begin{equation} \label{eq:final-union-bound}
    \textstyle 
    \sum_{k=1}^{N_h} \binom{N_h}{k} \, p_k =
    \sum_{k=1}^{T-1} \binom{N_h}{k} \, p_k +
    \sum_{k=T}^{N_h} \binom{N_h}{k} \, p_k
\end{equation}
From \Cref{lem:all-families-ub}, we can prove the following. %
\begin{restatable}{lemma}{twopartsofsum} \label{lem:two_parts_of_sum}
    It holds $\sum_{k=1}^{T-1} \binom{N_h}{k} p_k \le e^{-\Omega(d)}$ and $\sum_{k=T}^{N_h} \binom{N_h}{k} p_k \le e^{-\Omega(d)}$.
\end{restatable}
\noindent From \Cref{lem:two_parts_of_sum}, we get that \cref{eq:final-union-bound} is at most $e^{-\Omega(d)}$, proving that with probability at least $1 - e^{-\Omega(d)}$, no neuron-pruned subnetwork
$\varepsilon$-approximates the target ReLU uniformly on the ball of radius $R$.
This completes the proof of \Cref{thm:pruning-lower-bound}.

\section{Conclusion and Future Work}

This work establishes a fundamental separation between neuron pruning and weight
pruning within the framework of the Strong Lottery Ticket Hypothesis.
While weight pruning is known to $\eps$-approximate a single ReLU neuron with
logarithmic dependence on $1/\varepsilon$,
we show that neuron pruning necessitates
$\Omega(1/\varepsilon)$ hidden neurons.
Our results therefore identify neuron pruning as a strictly weaker pruning
mechanism from an approximation-theoretic perspective, even in the clean,
bias-free setting.

An interesting open question concerns the optimal dependence on the input
dimension $d$.
Our lower bound in \Cref{thm:pruning-lower-bound} scales linearly in $d$ (when $\varepsilon = \Omega(2^{-cd})$), whereas results for random feature models
suggest that approximation may in fact require exponentially many neurons in $d$
when doing neuron pruning.
We conjecture that a similar phenomenon should also hold for our simple neuron pruning scenario, and that even
the trivial strategy of retaining a single neuron is asymptotically optimal, as a function of the dimension.
For this simple setting, in which only a single neuron can be retained, we provide a proof of an exponential lower bound in $d$ in \Cref{apx:single-neuron-approx}.
Establishing a dimension-dependent lower bound would further clarify the
limitations of neuron pruning and strengthen the separation with weight pruning.
We leave this question, as well as extensions to deeper architectures and other
activation functions, for future work.

\bibliographystyle{plainnat}
\bibliography{biblio_arxiv}

\begin{thebibliography}{22}
\providecommand{\natexlab}[1]{#1}
\providecommand{\url}[1]{\texttt{#1}}
\expandafter\ifx\csname urlstyle\endcsname\relax
  \providecommand{\doi}[1]{doi: #1}\else
  \providecommand{\doi}{doi: \begingroup \urlstyle{rm}\Url}\fi

\bibitem[Burkholz(2022{\natexlab{a}})]{burkholzConvolutionalResidualNetworks2022}
Rebekka Burkholz.
\newblock Convolutional and {{Residual Networks Provably Contain Lottery
  Tickets}}.
\newblock In \emph{Proceedings of the 39th {{International Conference}} on
  {{Machine Learning}}}, pages 2414--2433, Baltimore, July 2022{\natexlab{a}}.
  PMLR.

\bibitem[Burkholz(2022{\natexlab{b}})]{burkholzMostActivationFunctions2022}
Rebekka Burkholz.
\newblock Most {{Activation Functions Can Win}} the {{Lottery Without Excessive
  Depth}}.
\newblock In \emph{Thirty-Sixth {{Conference}} on {{Neural Information
  Processing Systems}}}, December 2022{\natexlab{b}}.

\bibitem[Burkholz et~al.(2022)Burkholz, Laha, Mukherjee, and
  Gotovos]{burkholzExistenceUniversalLottery2022}
Rebekka Burkholz, Nilanjana Laha, Rajarshi Mukherjee, and Alkis Gotovos.
\newblock On the {{Existence}} of {{Universal Lottery Tickets}}.
\newblock In \emph{International {{Conference}} on {{Learning
  Representations}}}, virtual, April 2022.

\bibitem[{da Cunha} et~al.(2022){da Cunha}, Natale, and
  Viennot]{dacunhaProvingStrongLottery2022}
Arthur {da Cunha}, Emanuele Natale, and Laurent Viennot.
\newblock Proving the {{Strong Lottery Ticket Hypothesis}} for {{Convolutional
  Neural Networks}}.
\newblock In \emph{{{ICLR}} 2022 - 10th {{International Conference}} on
  {{Learning Representations}}}, Virtual, France, April 2022.

\bibitem[da~Cunha et~al.(2023)da~Cunha, D'Amore, and
  Natale]{cunhaPolynomiallyOverParameterizedConvolutional2023}
Arthur da~Cunha, Francesco D'Amore, and Emanuele Natale.
\newblock Polynomially {{Over-Parameterized Convolutional Neural Networks
  Contain Structured Strong Winning Lottery Tickets}}.
\newblock In \emph{Thirty-Seventh {{Conference}} on {{Neural Information
  Processing Systems}}}, November 2023.

\bibitem[Ferbach et~al.(2022)Ferbach, Tsirigotis, Gidel, and
  Bose]{ferbachGeneralFrameworkProving2022}
Damien Ferbach, Christos Tsirigotis, Gauthier Gidel, and Joey Bose.
\newblock A {{General Framework For Proving The Equivariant Strong Lottery
  Ticket Hypothesis}}.
\newblock In \emph{The {{Eleventh International Conference}} on {{Learning
  Representations}}}, September 2022.

\bibitem[Fischer et~al.(2022)Fischer, Gadhikar, and
  Burkholz]{fischerLotteryTicketsNonzero2022}
Jonas Fischer, Advait Gadhikar, and Rebekka Burkholz.
\newblock Lottery {{Tickets}} with {{Nonzero Biases}}, June 2022.

\bibitem[Frankle and Carbin(2018)]{frankleLotteryTicketHypothesis2018}
Jonathan Frankle and Michael Carbin.
\newblock The {{Lottery Ticket Hypothesis}}: {{Finding Sparse}}, {{Trainable
  Neural Networks}}.
\newblock In \emph{International {{Conference}} on {{Learning
  Representations}}}, September 2018.

\bibitem[He and Xiao(2023)]{HeX23}
Yang He and Lingao Xiao.
\newblock Structured pruning for deep convolutional neural networks: {A}
  survey.
\newblock \emph{CoRR}, abs/2303.00566, 2023.
\newblock \doi{10.48550/arXiv.2303.00566}.
\newblock URL \url{https://doi.org/10.48550/arXiv.2303.00566}.

\bibitem[Hoefler et~al.(2021)Hoefler, Alistarh, Ben{-}Nun, Dryden, and
  Peste]{HoeflerABDP21}
Torsten Hoefler, Dan Alistarh, Tal Ben{-}Nun, Nikoli Dryden, and Alexandra
  Peste.
\newblock Sparsity in deep learning: Pruning and growth for efficient inference
  and training in neural networks.
\newblock \emph{J. Mach. Learn. Res.}, 22:\penalty0 241:1--241:124, 2021.
\newblock URL \url{http://jmlr.org/papers/v22/21-0366.html}.

\bibitem[Hsu et~al.(2021)Hsu, Sanford, Servedio, and
  Vlatakis{-}Gkaragkounis]{hsuApproximationTwoLayer2021}
Daniel Hsu, Clayton Sanford, Rocco~A. Servedio, and Emmanouil~V.
  Vlatakis{-}Gkaragkounis.
\newblock On the approximation power of two-layer networks of random relus.
\newblock In Mikhail Belkin and Samory Kpotufe, editors, \emph{Conference on
  Learning Theory, {COLT} 2021, 15-19 August 2021, Boulder, Colorado, {USA}},
  volume 134 of \emph{Proceedings of Machine Learning Research}, pages
  2423--2461. {PMLR}, 2021.
\newblock URL \url{http://proceedings.mlr.press/v134/hsu21a.html}.

\bibitem[Kamath et~al.(2020)Kamath, Montasser, and
  Srebro]{DBLP:conf/colt/KamathMS20}
Pritish Kamath, Omar Montasser, and Nathan Srebro.
\newblock Approximate is good enough: Probabilistic variants of dimensional and
  margin complexity.
\newblock In \emph{Conference on Learning Theory}, pages 2236--2262. PMLR,
  2020.

\bibitem[Lueker(1998)]{Lueker98}
{George S.} Lueker.
\newblock Exponentially small bounds on the expected optimum of the partition
  and subset sum problem.
\newblock \emph{Random Structures and Algorithms}, 12:\penalty0 51--62, 1998.

\bibitem[Malach et~al.(2020)Malach, Yehudai, {Shalev-shwartz}, and
  Shamir]{malachProvingLotteryTicket2020}
Eran Malach, Gilad Yehudai, Shai {Shalev-shwartz}, and Ohad Shamir.
\newblock Proving the lottery ticket hypothesis: Pruning is all you need.
\newblock In \emph{Proceedings of the 37th {{International Conference}} on
  {{Machine Learning}}}, {{ICML}}'20, pages 6682--6691. JMLR.org, July 2020.

\bibitem[Mozer and Smolensky(1988)]{MozerS88}
Michael~C. Mozer and Paul Smolensky.
\newblock Skeletonization: a technique for trimming the fat from a network via
  relevance assessment.
\newblock In \emph{Proceedings of the 2nd International Conference on Neural
  Information Processing Systems}, NIPS'88, page 107–115, Cambridge, MA, USA,
  1988. MIT Press.

\bibitem[Mozer and Smolensky(1989)]{MozerS89}
Michael~C Mozer and Paul Smolensky.
\newblock Using relevance to reduce network size automatically.
\newblock \emph{Connection Science}, 1\penalty0 (1):\penalty0 3--16, 1989.

\bibitem[Orseau et~al.(2020)Orseau, Hutter, and
  Rivasplata]{orseauLogarithmicPruningAll2020}
Laurent Orseau, Marcus Hutter, and Omar Rivasplata.
\newblock Logarithmic pruning is all you need.
\newblock In \emph{Proceedings of the 34th {{International Conference}} on
  {{Neural Information Processing Systems}}}, {{NIPS}}'20, pages 2925--2934,
  Red Hook, NY, USA, December 2020. Curran Associates Inc.
\newblock ISBN 978-1-71382-954-6.

\bibitem[Pensia et~al.(2020)Pensia, Rajput, Nagle, Vishwakarma, and
  Papailiopoulos]{Pensia}
Ankit Pensia, Shashank Rajput, Alliot Nagle, Harit Vishwakarma, and Dimitris
  Papailiopoulos.
\newblock Optimal lottery tickets via subsetsum: logarithmic
  over-parameterization is sufficient.
\newblock In \emph{Proceedings of the 34th International Conference on Neural
  Information Processing Systems}, NIPS '20, Red Hook, NY, USA, 2020. Curran
  Associates Inc.
\newblock ISBN 9781713829546.

\bibitem[Ramanujan et~al.(2020)Ramanujan, Wortsman, Kembhavi, Farhadi, and
  Rastegari]{ramanujanWhatHiddenRandomly2020}
Vivek Ramanujan, Mitchell Wortsman, Aniruddha Kembhavi, Ali Farhadi, and
  Mohammad Rastegari.
\newblock What's {{Hidden}} in a {{Randomly Weighted Neural Network}}?
\newblock In \emph{2020 {{IEEE}}/{{CVF Conference}} on {{Computer Vision}} and
  {{Pattern Recognition}} ({{CVPR}})}, pages 11890--11899, June 2020.
\newblock \doi{10.1109/CVPR42600.2020.01191}.

\bibitem[Wang et~al.(2020)Wang, Zhang, Xie, Zhou, Su, Zhang, and
  Hu]{WangZXZSZH20}
Yulong Wang, Xiaolu Zhang, Lingxi Xie, Jun Zhou, Hang Su, Bo~Zhang, and Xiaolin
  Hu.
\newblock Pruning from scratch.
\newblock \emph{Proceedings of the AAAI Conference on Artificial Intelligence},
  34\penalty0 (07):\penalty0 12273--12280, Apr. 2020.
\newblock \doi{10.1609/aaai.v34i07.6910}.
\newblock URL \url{https://ojs.aaai.org/index.php/AAAI/article/view/6910}.

\bibitem[Yehudai and Shamir(2019)]{yehudaiShamirRandomFeatures}
Gilad Yehudai and Ohad Shamir.
\newblock On the power and limitations of random features for understanding
  neural networks.
\newblock \emph{CoRR}, abs/1904.00687, 2019.
\newblock URL \url{http://arxiv.org/abs/1904.00687}.

\bibitem[Zhou et~al.(2019)Zhou, Lan, Liu, and
  Yosinski]{zhouDeconstructingLotteryTickets2019}
Hattie Zhou, Janice Lan, Rosanne Liu, and Jason Yosinski.
\newblock Deconstructing {{Lottery Tickets}}: {{Zeros}}, {{Signs}}, and the
  {{Supermask}}.
\newblock In \emph{Advances in Neural Information Processing Systems 32: Annual
  Conference on Neural Information Processing Systems 2019 (NIPS 2019)}, page
  3592–3602, 2019.

\end{thebibliography}
\newpage

\appendix
\crefalias{section}{appendix}

\section{Additional Figures}
\label{apx:figures}
\usetikzlibrary{fit}

\begin{figure}[h]
\centering
\begin{tikzpicture}[x=0.65cm,y=0.85cm,>=stealth]

\def\T{10}   %
\def\posA{2}
\def\posB{6}
\def\posC{9}
\def\posD{5}

\tikzset{
  state/.style={circle,draw,inner sep=1.2pt},
  marker/.style={circle,fill,inner sep=2.0pt},
  line/.style={draw,thick},
  label/.style={anchor=east,font=\small}
}

\newcommand{\drawwalkline}[4]{%
  \node[label] at (-1.2,#1) {#2};

  \draw[line] (0,#1) -- (\T,#1);

  \node[state] (zero-#4) at (0,#1) {};
  \node[state] (Tend-#4) at (\T,#1) {};

  \foreach \s in {1,...,\numexpr\T-1\relax} {
    \node[state] at (\s,#1) {};
  }

  \node[font=\small,anchor=south] at (0,#1+0.35) {$0$};
  \node[font=\small,anchor=south] at (\T,#1+0.35) {$T$};

  \node[marker] at (#3,#1) {};
}

\drawwalkline{0}{Input family $i=1$}{\posA}{A}
\drawwalkline{-1}{Input family $i=2$}{\posB}{B}

\node[font=\Large] at (\T/2,-2.0) {$\vdots$};

\drawwalkline{-3}{Input family $i=\lfloor d/2\rfloor-1$}{\posC}{C}
\drawwalkline{-4}{Input family $i=\lfloor d/2\rfloor$}{\posD}{D}

\node[draw=cyan,very thick,rounded corners,inner sep=4pt,fit=(zero-A) (zero-B) (zero-C) (zero-D)] (allzeros) {};

\node[font=\small,align=center] at (\T/2,-5.2)
{success requires\\ all chains at $0$};
\draw[->,cyan!70!black] (\T/2-2,-4.8) -- (allzeros.south);

\end{tikzpicture}
\caption{
Stacked state-line representation of the birth--death chain $(B^{\mathrm{bd}}_s)_{s\ge 0}$.
Each horizontal line is the state space $\{0,1,\dots,T\}$ for one input family, and the filled
dot marks the current value of $B^{\mathrm{bd}}_s$.
A necessary condition for $\varepsilon$-approximation is that all $\lfloor d/2\rfloor$
independent chains reach state $0$.
}
\label{fig:RWs}
\end{figure}

\section{Omitted Proofs}
\label{apx:omitted-proofs}

\brokenbinpreventsapprox*
\begin{proof}
    Since $t^\star \notin I_B$, the function $f$ is linear on $I_B$.
    By \Cref{def:broken-bin}, for any affine function $\ell$,
    there exist three points in $I_B$ on which $h$ differs from $\ell$
    by at least $c\varepsilon$.
    Taking $\ell = f$ yields the claim.
\end{proof}

\begin{figure}[t]
    \centering
    \begin{tikzpicture}

        \def\tstar{-0.8}    %
        \def\tneu{1.7}      %
        \def\mt{0.95}       %
        \def\mn{0.55}       %
        \def\epsw{0.7}      %
        \def\tprobe{\tstar+\epsw} %

        \draw[->] (-3.6,0) -- (3.5,0) node[right] {$t$};
        \draw[->] (1,-0.4) -- (1,3.5) node[above] {$y$};

        \draw[red,very thick,opacity=0.75] (-3.6,0) -- (\tstar,0);
        \draw[red,very thick,opacity=0.75] (\tstar,0) -- (3.2,{(\mt)*(3.2-\tstar)});

        \draw[cyan,very thick,opacity=0.75] (-3.6,0) -- (\tneu,0);
        \draw[cyan,very thick,opacity=0.75] (\tneu,0) -- (3.2,{(\mn)*(3.2-\tneu)});

        \draw[dashed] (\tstar,0) -- (\tstar,3.35);
        \node[below] at (\tstar,0) {$t_i^\star$};

        \draw[cyan,dashed] (\tneu,0) -- (\tneu,3.35);
        \node[below,cyan] at (\tneu,0) {$t_{i,j}$};

        \draw[decorate,decoration={brace,mirror}]
        (\tstar-\epsw,-0.55) -- (\tstar+\epsw,-0.55)
        node[midway,below=1pt] {\small $\varepsilon$-window};

        \draw[<->,thick]
            (\tprobe,{(\mt)*(\tprobe-\tstar)}) -- (\tprobe,{max(0,(\mn)*(\tprobe-\tneu))})
            node[midway,right=-1pt] {\small error};

        \node[anchor=west, red]  at (2.7,3.05) {\small target ReLU(red)};
        \node[anchor=west, cyan] at (2.7,1.2) {\small neuron ReLU(cyan)};

    \end{tikzpicture}
    \caption{Breakpoint alignment intuition along a one-dimensional input family $\mathbf{x}_i(t)$.
    Along $\mathbf{x}_i(t)$, a bias-free ReLU neuron $\sigma(\langle w,x\rangle)$ reduces to the one-dimensional function $t \mapsto \sigma(w_{2i-1}t + w_{2i})$, whose breakpoint is $t_i(w) = -w_{2i}/w_{2i-1}$.
    The target (red) has breakpoint $t_i^\star$, while a randomly drawn neuron (cyan) has breakpoint $t_{i,j}$.
    If $|t_{i,j}-t_i^\star|>\varepsilon$, then on the $\varepsilon$-neighborhood of $t_i^\star$ the cyan function is linear (in the picture, flat) whereas the target changes slope, yielding a nontrivial uniform error, as stated in \Cref{lem:breakpoint-necessary}.}
    \label{fig:breakpoint}
\end{figure}

\breakpointnecessary*
\begin{proof}
    By translation, we may assume $t^\star=0$, so that
    $f(t)=\sigma(a t)$.
    Let $h(t)=\alpha t+\beta$ be any linear function.
    Consider the three points $t=-\delta,\,0,\,\delta$.
    At these points,
    $
        f(-\delta)=0,$ $
        f(0)=0,$ $
        f(\delta)=a\delta$ $\text{if } a>0,$
    and the roles of $\pm\delta$ are reversed if $a<0$.
    In either case, $|f(\delta)-f(-\delta)|=|a|\delta$.
    Define the errors
   $
        \Delta_-  := h(-\delta)-f(-\delta), 
        \Delta_0  := h(0)-f(0),$ $      
        \Delta_+  := h(\delta)-f(\delta).
    $
    Since $h$ is linear,
    $
        h(0)=\tfrac12\bigl(h(-\delta)+h(\delta)\bigr),
    $
    which implies
    $
        \Delta_0
        =
        \tfrac12\bigl(\Delta_-+\Delta_+ + |a|\delta\bigr)
    $
    (up to relabeling the endpoints when $a<0$).

    If $|\Delta_-|<\tfrac{|a|\delta}{4}$ and
    $|\Delta_+|<\tfrac{|a|\delta}{4}$, then
    $
        \Delta_0
        >
        \tfrac12\!\left(-\tfrac{|a|\delta}{4}-\tfrac{|a|\delta}{4}
        +|a|\delta\right)
        =
        \tfrac{|a|\delta}{4},
    $
    so $|\Delta_0|\ge \tfrac{|a|\delta}{4}$.
    Therefore, at least one of
    $|\Delta_-|,|\Delta_0|,|\Delta_+|$
    is at least $\tfrac{|a|\delta}{4}$, and hence
    \[
        \sup_{t\in[-\delta,\delta]}|f(t)-h(t)|
        \;\ge\;
        \tfrac{|a|\delta}{4}.
    \]

    Taking $c=\nicefrac{|a|}{4}$ concludes the proof.
\end{proof}

\singlebreakpointbreaksbin*
\begin{proof}
    Let $\Delta=\varepsilon/4$ and define three points
    $
        t_1 := t_0-\Delta, 
        t_2 := t_0, 
        t_3 := t_0+\Delta,
    $
    which all lie in $B$.
    Since $h$ is linear on each side of $t_0$, we have
    \[
        \frac{h(t_2)-h(t_1)}{t_2-t_1}=m_-,
        \qquad
        \frac{h(t_3)-h(t_2)}{t_3-t_2}=m_+.
    \]

    Let $\ell(t)=at+b$ be any affine function.
    Assume that $\ell$ fits the three points with error at most $\eta\varepsilon$, i.e.,
    $
        |h(t_j)-\ell(t_j)| \le \eta\varepsilon
        $ $\text{for } j=1,2,3.
    $
    We compare slopes. Since $\ell$ is affine,
    \[
        a=\frac{\ell(t_2)-\ell(t_1)}{t_2-t_1}.
    \]
    Thus,
    \[
        |a-m_-|
        =
        \left|
        \frac{(\ell(t_2)-h(t_2))-(\ell(t_1)-h(t_1))}{t_2-t_1}
        \right|
        \le
        \frac{2\eta\varepsilon}{\Delta}
        =8\eta.
    \]
    The same argument applied to $(t_2,t_3)$ gives
    $
        |a-m_+|\le 8\eta.
    $
    By the triangle inequality,
    $
        |m_+-m_-|
        \le
        |m_+-a|+|a-m_-|
        \le
        16\eta.
    $
    Since $|m_+-m_-|\ge\gamma$, this implies $\eta\ge\gamma/16$.
    Therefore, for every affine function $\ell$,
    \[
        \max_{j=1,2,3}|h(t_j)-\ell(t_j)|
        \ge
        \frac{\gamma}{16}\,\varepsilon.
    \]
    Taking $c=\gamma/16$ concludes the proof.
\end{proof}

\singlefamilyub*
\begin{proof}
    Let $R_k$ and $L_k$ denote the number of right moves and left moves, respectively,
    during the $k$ steps of the process. 
    A right move happens when a new broken bin is introduced, and a left move when it is removed.
    Then $B_k^{\mathrm{bd}} = 1 + R_k - L_k$, which means that in order to have $B_k^{\mathrm{bd}} = 0$ we must have $L_k > R_k$. By the law of total probability, we thus have
    \begin{align}
        \Pr(B_k^{\mathrm{bd}} = 0)
        & = \Pr(1 + R_k = L_k) \nonumber\\
         & \leq \Pr(L_k > R_k)                                                                             \nonumber      \\
         & = \Pr\left(L_k > R_k \mid R_k < p k / 2\right)\Pr(R_k < p k / 2)  +         \nonumber                       \\
         & \qquad + \Pr\left(L_k > R_k \mid R_k \geq p k / 2\right)\Pr(R_k \geq p k / 2)    \nonumber                  \\
         & \leq \Pr(R_k < p k / 2) + \Pr\left(L_k > R_k \mid R_k \geq p k / 2\right)     \nonumber                     \\
         & \leq \Pr(R_k < p k / 2) + \Pr\left(L_k \ge p k / 2 \mid R_k \geq p k / 2\right)     \nonumber                     \\
         & \leq \Pr(R_k < p k / 2) + \Pr\left(L_k \ge p k / 2 \right),         \label{eq:necessary_steps_left_small_k}
    \end{align}
    where, in the last inequality, we used that conditioning on performing at least a certain number of right moves can only decrease the probability of performing at least a given number of left moves, since those moves cannot be to the left.
    More formally, if we define $h(r) := \Pr(L_k\ge pk/2 \mid R_k=r)$ we have
    \begin{equation*}
        \Pr\!\left(L_k \ge \tfrac{pk}{2}\!\right) = \sum_{r=0}^k h(r) \Pr(R_k=r) = \mathbb{E}[h(R_k)],
    \end{equation*}
    and, similarly,
    \begin{equation*}
    \Pr\!\left(L_k \ge \tfrac{pk}{2} \mid R_k \ge \tfrac{pk}{2}\right)
    = \sum_{r=\frac{pk}{2}}^k h(r)\Pr\!\left(R_k=r \mid R_k \ge \tfrac{pk}{2}\right)
    = \mathbb{E}\!\left[h(R_k) \mid R_k \ge \tfrac{pk}{2}\right].
    \end{equation*}
    Now, since $h(r)$ is a decreasing function of $r$, it follows that
    \begin{equation*}
        \mathbb{E}\!\left[h(R_k) \mid R_k \geq \tfrac{pk}{2}\!\right] \le \mathbb{E}[h(R_k)]
        \; \iff \; 
        \Pr\!\left(L_k \ge \tfrac{pk}{2} \mid R_k \ge \tfrac{pk}{2}\right) 
        \leq \Pr\!\left(L_k \ge \tfrac{pk}{2}\!\right),
    \end{equation*}
    justifying the last step in \cref{eq:necessary_steps_left_small_k}.

    Intuitively, \cref{eq:necessary_steps_left_small_k} states that the probability that the process ends in a state $B_k^{\mathrm{bd}} = 0$ can be bounded by the event that the number of right moves is too small, or that the number of left moves is at least the number of right ones.

    \textbf{Case 1: $k \leq 2T/3$.}  
   Since $k \le 2T/3$ the cap at $T$ of the birth-death process  is never reached.
    Since each step produces a right move independently with probability $p$, then
    $\mathbb{E}[R_k] = pk$, and by a Chernoff bound $\Pr(R_k < p k / 2) \leq e^{-\Omega(k)}$.
    On the other hand, each step produces a left move with probability at most
    $q \le p/2 $ is a constant.
    Hence $\mathbb{E}[L_k] \le q_0 k$.
    Applying a Chernoff bound to $L_k$, we obtain $\Pr\!\left(L_k \ge \tfrac{p}{2} k\right)\leq e^{-\Omega(k)}$, hence
    \begin{equation}
        \label{eq:left-steps}
        \Pr(R_k < p k / 2) + \Pr\left(L_k \ge p k / 2 \right)
        \le\;
        e^{-\Omega(k)}.
    \end{equation}

    \textbf{Case 2: $k > 2T/3$.} At step $\hat k = k-2T/3$, there are two possibilities: either $B_{\hat k}^{\mathrm{bd}} \leq T/3$ or $B_{\hat k}^{\mathrm{bd}} > T/3$.
    If $B_{\hat k}^{\mathrm{bd}} \leq T/3$, then we can use the same argument as in Case 1 to bound the probability of success, as the cap at $T$ of the birth-death process will not be reached after $k$ steps.
    If $B_{\hat k}^{\mathrm{bd}} > T/3$, then the probability of reaching state 0 in the remaining $2T/3$ is bounded by the probability that the process does at least $T/3$ left moves, which is at most $e^{-\Omega(T)}$. Formally, we have
    \begin{align}
        \Pr(B_k^{\mathrm{bd}} = 0)
         & \leq \Pr(L_k > R_k) \nonumber                                                                                             \\
         & = \Pr\left(L_k > R_k \mid B_{\hat k}^{\mathrm{bd}} \leq T/3\right)\Pr(B_{\hat k}^{\mathrm{bd}} \leq T/3)  +  \nonumber \\
         & \qquad + \Pr\left(L_k > R_k \mid B_{\hat k}^{\mathrm{bd}} > T/3\right)\Pr(B_{\hat k}^{\mathrm{bd}} > T/3) \nonumber    \\
         & \leq \Pr\left(L_k > R_k\mid B_{\hat k}^{\mathrm{bd}} \leq T/3\right)
        + \Pr\left(L_k > R_k \mid B_{\hat k}^{\mathrm{bd}} > T/3\right) \nonumber                                                 \\
         & \leq \Pr\left(L_k - L_{\hat k} \geq R_k - R_{\hat k}\mid B_{\hat k}^{\mathrm{bd}} \leq T/3\right)
        + \Pr\left(L_k - L_{\hat k} > T / 3 \mid B_{\hat k}^{\mathrm{bd}} > T/3\right) \nonumber                                                 .         \label{eq:necessary_steps_left_large_k}
    \end{align}
    The first term in the last line are bounded by $e^{-\Omega(T)}$ as shown in Case 1 using similar arguments: note that starting at position larger than 1, only makes it less likely to reach $0$.
    The second term is bounded by a standard Chernoff bound. 

    Putting both cases together, we get 
    \[
        \Pr\left(B_k^{\mathrm{bd}} = 0 \right)
        \le e^{-\Omega(\min(k,T))}.
    \]
    \end{proof}

\twopartsofsum*
\begin{proof}
    For $k \le T$, from \Cref{lem:all-families-ub} we get that
$
    p_k \;\le\; e^{-c_4 k d}
$
for some constant $c_4>0$.
Using the bound $\binom{N_h}{k} \le (eN_h/k)^k$, we obtain
\begin{equation} 
    \label{eq:small-k}
     \sum_{k=1}^{T-1} \binom{N_h}{k} \, p_k
    \;\le\;
     \sum_{k=1}^{T-1} e^{-
    k \bigl(c_4 d - \log N_h + O(1)\bigr) }  
    \stackrel{(a)}{\leq} \sum_{k=1}^{T-1}  e^{-\Omega(kd)} 
    \stackrel{(b)}{\leq} e^{-\Omega(d)} , 
\end{equation}
where in $(a)$ we used that $\log N_h \leq  cd$, 
and in $(b)$ we upper bounded $\sum_{k=1}^{T-1} e^{-\Omega(kd)} \leq \sum_{k=1}^{\infty} (e^{-\Omega(d)})^k $ and used that 
$\sum_{k=1}^{\infty} x^k = \frac{x}{1-x} $. 

For $k > T$, given \Cref{lem:all-families-ub} and the choice of $T$ in \cref{eq:cap-choice}, we get
$
    p_k \;\le\; e^{-c_5 d/\varepsilon}
$
for some constant $c_5>0$.
Using the crude bound $\sum_{k>k_0} \binom{N_h}{k} \le 2^{N_h}$, we obtain
\begin{equation}
\label{eq:big-k}
    \sum_{k=T}^{N_h} \binom{N_h}{k} \, p_k
    \;\le\;
    2^{N_h} e^{-c_5 d/\varepsilon}
    \leq  e^{-\Omega(d/\varepsilon)} 
    \leq e^{-\Omega(d)},
\end{equation}
where in the second to last inequality we used the hypothesis $N_h \leq c d/\varepsilon$.
\end{proof}

\section{Extension of \Cref{thm:pruning-lower-bound}}
\label{sec:appendix-output-weights}

We briefly explain how the proof of \Cref{thm:pruning-lower-bound} extends to the case where the output weights are not fixed to $1$, but have magnitude bounded away from $0$. Consider the network
$
g(\mathbf{x})=\sum_{j=1}^{N_h}\alpha_j\,\sigma(\langle \mathbf{w}_j,\mathbf{x}\rangle),
$
where the hidden weights $\mathbf{w}_j\sim\mathcal N(0,I_d)$ are independent, and where the coefficients $\alpha_j$ (deterministic or random) satisfy $|\alpha_j|\ge a_-$, for some constant $a_->0$ and all $j$.

The proof proceeds exactly as in the unit-weight case after observing that, along any input family $\mathbf{x}_i(t)$, the restriction of the network is
$
g(\mathbf{x}_i(t))=\sum_{j\in S}\alpha_j\,\sigma(w_{j,2i-1}t+w_{j,2i}).
$
The breakpoint locations remain
$
t_{i,j}=-w_{j,2i}/w_{j,2i-1},
$
so they do not depend on the coefficients $\alpha_j$. The only change is that the slope jump at such a breakpoint is now multiplied by $\alpha_j$, namely
$
\Delta_{i,j}=\alpha_j w_{j,2i-1}.
$
Since $|\alpha_j|\ge a_-$, we have $|\Delta_{i,j}|\ge a_-|w_{j,2i-1}|$, so every estimate in the proof that uses a lower bound on the slope jump still goes through with constants depending on $a_-$.

In particular, the breakpoint-based lemma, the construction of the capped and birth--death processes, and the final counting/union-bound argument are unchanged up to constant factors. The only effect of replacing unit output weights by coefficients bounded away from zero is therefore a modification of the numerical constants in the proof. This yields the same lower bound on the required width, with constants depending on $a_-$.

\section{Proof of \Cref{thm:pruning-lower-bound-log}}
\label{apx:proof-lower-bound-log}

\pruninglowerboundlog*
\begin{proof}
Because the Gaussian distribution $\mathcal{N}(0,I_d)$ and the input domain $\{\mathbf x : \|\mathbf x\|_2 \le R\}$ are perfectly symmetric, the problem is invariant to orthogonal transformations. Therefore, we can rotate the coordinate system so that the vector $\mathbf{w}^\star$ aligns with the first standard basis vector. Thus, without loss of generality, we assume $\mathbf{w}^\star = \sqrt{d} \mathbf{e}_1$.

For each $i\in\{2,\dots,d\}$, define the one-dimensional input family
$
\mathbf x_i(t):=t\,\mathbf e_1+\mathbf e_i, t\in\mathbb R
$.
Along this path, the output of the target neuron is
$
f(\mathbf x_i(t))=\sigma(\langle \sqrt{d} \mathbf e_1,\mathbf x_i(t)\rangle)=\sigma(\sqrt{d}t)
$,
so its breakpoint is at $t^\star=0$. 
Note that despite the shared notation, these are distinct from the input families considered in the proof of \Cref{thm:pruning-lower-bound}.
Let $\eps' = \eps/\sqrt{d}$,
and let
$
I:= [-\eps',\eps']
$.
By Lemma~\ref{lem:breakpoint-necessary}, applied to $f(t)=\sigma(\sqrt{d}t)$ and $I$, any function that is linear on an interval containing $I$ cannot approximate $\sigma(\sqrt{d}t)$ on $I$ with error smaller than $\eps / 4$.

Now fix a subset $S\subseteq\{1,\dots,N_h\}$ and consider the pruned network
$
g_S(\mathbf x)=\sum_{j\in S}\alpha_j\,\sigma(\langle \mathbf w_j,\mathbf x\rangle)
$,
where each hidden neuron
$j$ has an incoming weight vector $\mathbf{w_j}=(w_{j,1},\dots,w_{j,d})$.
Its restriction to the $i$-th family $\mathbf x_i(t)$ is
\[
g_{S,i}(t):=g_S(\mathbf x_i(t))
=\sum_{j\in S}\alpha_j\,\sigma(w_{j,1}t+w_{j,i}),
\]
whose breakpoints are exactly
$t_{i,j}:=-\nicefrac{w_{j,i}}{w_{j,1}}$, $j\in S$.
If none of these breakpoints falls in $I$, then $g_{S,i}$ is linear on $I$ and cannot $(\varepsilon/4)$-approximate $\sigma(\sqrt{d}t)$ there.
Hence, for each input family $i\in\{2,\dots,d\}$, successful approximation implies the event
\[
E_i:=\Bigl\{\exists\,j\in\{1,\dots,N_h\}\text{ such that }t_{i,j}\in I\Bigr\}.
\]
In particular, overall success implies $\bigcap_{i=2}^d E_i$.
We now bound the probability of these events. Let $Z_j:=w_{j,1}$ and write
\[
L_1:=\sum_{j=1}^{N_h}|Z_j| \qquad \text{and} \qquad L_2:=\sum_{j=1}^{N_h}Z_j^2.
\]
Intuitively, by controlling the sums \(L_1=\sum_j |Z_j|\) and
\(L_2=\sum_j Z_j^2\), we rule out the possibility that the shared denominators are
collectively large enough to make breakpoint hits likely across all input families.

For fixed $i$ and $j$, conditional on $Z_j=z$, the breakpoint condition $t_{i,j}\in I$ is equivalent to $|w_{j,i}|\le \varepsilon' |z|$. 
Because $w_{j,i}\sim\mathcal N(0,1)$, the probability of the complementary event $\Pr(|w_{j,i}| > x)$ equals $1$ at $x=0$ and decays as roughly $e^{-x^2/2}$ for large $x$. Consequently, there exist universal constants $a, b > 0$ such that
\[
\Pr(|w_{j,i}| > x) \ge \exp(-a x - b x^2) \qquad \text{for all } x \ge 0.
\]
Thus, the conditional probability that a single unit has no breakpoint in $I$ is bounded below by
\[
\Pr(t_{i,j}\notin I\mid Z_j=z) \ge \exp(-a \varepsilon' |z| - b \varepsilon'^2 z^2).
\]
Consequently, conditioning on $Z:=(Z_1,\dots,Z_{N_h})$ and using independence across neurons, we can bound the probability of $E_i$ by:
\[
\Pr(E_i\mid Z) = 1 - \prod_{j=1}^{N_h} \Pr(t_{i,j}\notin I\mid Z_j) \le 1-\exp\Bigl(-a\varepsilon' L_1 - b\varepsilon'^2 L_2\Bigr).
\]
Moreover, conditioned on $Z$, the events $E_2,\dots,E_d$ are independent, because they depend on disjoint coordinate families $\{w_{j,i}\}_{j=1}^{N_h}$ for different $i\ge 2$. Therefore,
\[
\Pr(\text{success}\mid Z)
\le \Bigl(1-\exp(-a\varepsilon' L_1 - b\varepsilon'^2 L_2)\Bigr)^{d-1}.
\]

It remains to control $L_1$ and $L_2$. Since $Z_j \sim \mathcal{N}(0,1)$, we have $\mathbb{E}[|Z_j|] = \sqrt{2/\pi} < 1$ and $\mathbb{E}[Z_j^2] = 1$, so $\mathbb{E}[L_1] < N_h$ and $\mathbb{E}[L_2] = N_h$. By Markov's inequality, 
\[
\Pr\left(L_1 > \frac{4 N_h}{\delta}\right) \le \frac{\mathbb{E}[L_1]}{4 N_h / \delta} < \frac{\delta}{4}, \qquad \text{and} \qquad \Pr\left(L_2 > \frac{4 N_h}{\delta}\right) \le \frac{\mathbb{E}[L_2]}{4 N_h / \delta} = \frac{\delta}{4}.
\]
By a union bound on these two events, with probability at least $1 - \delta/2$, both $L_1 \le K N_h$ and $L_2 \le K N_h$ hold, where $K = 4/\delta$. On this high-probability event, and using $\varepsilon' \in (0,1)$, we have
\[
a\varepsilon' L_1 + b\varepsilon'^2 L_2 \le a\varepsilon' K N_h + b\varepsilon' K N_h = (a+b)K N_h \varepsilon'.
\]
Let $c_1 = a+b$. We then have
\[
\Pr(\text{success}\mid Z)
\le \Bigl(1-\exp(-c_1 K N_h\varepsilon')\Bigr)^{d-1}.
\]
Now choose $c_0>0$ small enough so that $c_1Kc_0\le 1/2$. If $N_h < c_0\frac{\log d}{\varepsilon'}$, then $\exp(-c_1 K N_h\varepsilon') \ge d^{-1/2}$, and hence for $d\geq 2$, 
\[
\Pr(\text{success}\mid Z)
\le (1-d^{-1/2})^{d-1}
\leq (e^{-d^{-1/2}})^{d-1}
\le e^{-\sqrt d/2}.
\]
For all sufficiently large $d$, this is at most $\delta/2$; the finitely many remaining dimensions can be absorbed by shrinking $c_0$ once more if needed. Therefore, using the total law of probability
\[
\Pr(\text{success})
\le 1\cdot \Pr(L_1 > K N_h \cup L_2 > K N_h)+\Pr(\text{success}\mid L_1, L_2 \le K N_h)\cdot 1
\le \delta.
\]
This proves the theorem.
\end{proof}

\section{ReLU Approximation with a Single Neuron}
\label{apx:single-neuron-approx}
\begin{lemma}[Single-neuron approximation requires exponential width]
    \label{lem:single-neuron-exponential}
    Given $\mathbf{w^\star} \in \mathbb{R}^d$, with $\|\mathbf{w^\star}\|_2 = 1$,
    let
    $
        f(\mathbf{x}) := \sigma\bigl(\langle \mathbf{w^\star}, x\rangle\bigr)
    $
    be a target ReLU neuron without bias.
    Let $\mathbf{w}$ be sampled uniformly at random from the unit sphere $\mathbb{S}^{d-1}$, and
    consider approximations of the form
    \[
        g(\mathbf{x}) := a\,\sigma\bigl(\langle \mathbf{w}, \mathbf{x}\rangle\bigr), \qquad a\in\mathbb{R}.
    \]
    There exist universal constants $c,C>0$ such that for any $\varepsilon\in(0,1)$,
    \[
        \Pr\!\left(
        \inf_{a\in\mathbb{R}}
        \sup_{\|\mathbf{x}\|_2\le 1}
        \bigl|f(\mathbf{x})-g(\mathbf{x})\bigr|
        \le
        \varepsilon
        \right)
        \;\le\;
        (C\varepsilon)^{c d}.
    \]
    Consequently, if a width-$N_h$ random network is initialized with i.i.d.\ weights
    $\mathbf{w_1},\dots,\mathbf{w_{N_h}}\sim\mathrm{Unif}(\mathbb{S}^{d-1})$ and only a single neuron may be
    retained, then achieving uniform $\varepsilon$-approximation of $f$ with non-negligible
    probability requires $N_h \ge e^{\Omega(d)}$.
\end{lemma}

\begin{proof}
    Fix $\varepsilon\in(0,1)$ and suppose that for some $a\in\mathbb{R}$,
    \begin{equation}
        \label{eq:uniform-approx}
        \sup_{\|x\|_2\le 1}
        \bigl|\sigma\bigl(\langle \mathbf{w^\star}, \mathbf{x}\rangle\bigr)
        -
        a\,\sigma\bigl(\langle \mathbf{w}, \mathbf{x}\rangle\bigr)\bigr|
        \;\le\;
        \varepsilon.
    \end{equation}

    Evaluating~\eqref{eq:uniform-approx} at $\mathbf{x}=\mathbf{w^\star}$ yields
    \[
        \bigl|1 - a\,\sigma\bigl(\langle \mathbf{w}, \mathbf{w^\star}\rangle\bigr)\bigr| \le \varepsilon.
    \]
    Since $\sigma(\langle \mathbf{w}, \mathbf{w^\star}\rangle)\le 1$, this implies $|a|\ge 1-\varepsilon$, and in
    particular $|a|=\Theta(1)$.
    Thus the scalar coefficient $a$ cannot compensate for a significant misalignment
    between $\mathbf{w}$ and $\mathbf{w^\star}$.

    Let $\theta\in[0,\pi]$ denote the angle between $\mathbf{w}$ and $\mathbf{w^\star}$, so that
    $
        \langle \mathbf{w^\star}, \mathbf{w}\rangle = \cos\theta.
    $
    Consider the two-dimensional subspace $\mathrm{span}\{\mathbf{w^\star},\mathbf{w}\}$, and choose a unit vector
    $\mathbf{x}$ in this plane such that
    $\langle \mathbf{w^\star}, \mathbf{x}\rangle = \sin\theta$ and $\langle \mathbf{w}, \mathbf{x}\rangle = 0$.
    Such a choice is always possible by elementary planar geometry.
    For this $\mathbf{x}$ we have
    $
        \sigma\bigl(\langle \mathbf{w^\star}, \mathbf{x}\rangle\bigr) = \sin\theta
    $ and
    $
        \sigma\bigl(\langle \mathbf{w}, \mathbf{x}\rangle\bigr) = 0
    $,
    and therefore
    \[
        |f(\mathbf{x})-g(\mathbf{x})| = \sin\theta.
    \]
    Combining with~\eqref{eq:uniform-approx} yields
    \[
        \sin\theta \le \varepsilon,
    \]
    which implies $\theta \le 2\varepsilon$ for all sufficiently small $\varepsilon$.

    Since $\mathbf{w}$ is uniformly distributed on $\mathbb{S}^{d-1}$, the probability that
    $\angle(\mathbf{w},\mathbf{w^\star})\le 2\varepsilon$ equals the normalized surface area of a spherical cap of
    half-angle $2\varepsilon$, where by  $\angle(\mathbf{w},\mathbf{w^\star})$ we denote the angle between the two vectors $\mathbf{w}$ and $\mathbf{w^\star}$.
    It is standard that there exist universal constants $c,C>0$ such that for all
    $\varepsilon\in(0,1)$,
    \[
        \Pr\bigl(\angle(\mathbf{w},\mathbf{w^\star})\le 2\varepsilon\bigr)
        \;\le\;
        (C\varepsilon)^{c(d-1)}.
    \]

    If $N_h$ neurons are sampled independently, a union bound shows that the probability
    that \emph{any} single neuron can $\varepsilon$-approximate $f$ is at most
    $N_h(C\varepsilon)^{c d}$.
    Thus, unless $N_h\ge e^{\Omega(d)}$, this probability vanishes exponentially in $d$.
\end{proof}
\end{document}